\documentclass{article}
\usepackage{spconf,amsmath,graphicx}
\usepackage{amsthm}
\usepackage{amssymb}
\usepackage{cite}
\usepackage{graphicx,epstopdf}
\usepackage{color}
\usepackage{times}
\usepackage{verbatim}
\usepackage{floatflt}
\usepackage{enumerate}
\usepackage{array}
\usepackage{multicol,afterpage,wrapfig}
\usepackage{tikz}
\usepackage{algorithm}
\usepackage[noend]{algpseudocode}

\usepackage[noend]{algcompatible}
\usepackage{algorithmicx}

\makeatletter
\def\BState{\State\hskip-\ALG@thistlm}
\makeatother
\usepackage{amsthm,amssymb,eucal}
\usepackage[utf8]{inputenc}
\usepackage{epsfig}
\usepackage{exscale}
\usepackage{latexsym}
\usepackage{verbatim}
\usepackage{amsfonts}
\usepackage{subfigure}
\usepackage{multirow}
\usepackage{booktabs}
\usepackage{adjustbox}

\usepackage{cite}
\usepackage{balance}
\usepackage{color}
\usepackage{transparent}
\usepackage{import}
\usepackage{mathtools}
\usepackage{tikz}
\usetikzlibrary{automata,arrows,positioning,calc}
\usepackage{bbm}
\usepackage[printonlyused,withpage]{acronym}
\usepackage{tabu}
\usepackage{longtable}
\usepackage{mathtools}
\usepackage{stfloats}
\usepackage{psfrag}
\usepackage{float}
\usepackage{acronym}  
\usepackage{psfrag}
\usepackage{adjustbox}

\graphicspath{{../Figures/}}
\DeclareGraphicsExtensions{.eps}

\usepackage{amssymb}
\usepackage{amsthm}

\DeclareMathOperator*{\argmin}{arg\,min}

\title{Personalizing Federated Learning with Over-the-Air Computations}

\name{Zihan Chen $^{\dagger *}$\thanks{$^*$Equal contribution. This work was supported in part by the National Natural Science Foundation of China under Grant 62271513. (\textit{Corresponding Author: Howard H. Yang})}  \qquad Zeshen Li$^{\ddagger *}$ \qquad Howard H. Yang$^{\ddagger}$ \qquad Tony Q.S. Quek$^{\dagger }$}
\address{$^{\dagger}$ Singapore University of Technology and Design, Singapore 487372  \\
         $^{\ddagger}$ ZJU-UIUC Institute, Zhejiang University, Haining 314400, China}

\begin{document}

\newtheorem{theorem}{Theorem}[section]
\newtheorem{proposition}[theorem]{Proposition}
\newtheorem{lemma}[theorem]{Lemma}
\newtheorem{corollary}[theorem]{Corollary}
\newtheorem{problem}[theorem]{Problem}
\newtheorem{conjecture}[theorem]{Conjecture}

\newtheorem{definition}[theorem]{Definition}
\newtheorem{example}[theorem]{Example}
\newtheorem{remark}[theorem]{Remark}
\newtheorem{assumption}[theorem]{Assumption}

\algnewcommand\algorithmicreturn{\textbf{return} }
\algnewcommand\RETURN{\State \algorithmicreturn}%

\algrenewcommand\algorithmicrequire{\textbf{function}} 
\algnewcommand\algorithmicreq{\textbf{Require:} }
\algnewcommand\REQ{\STATEx \algorithmicreq{}}%

\maketitle

\begin{abstract}

Federated edge learning is a promising technology to deploy intelligence at the edge of  wireless networks in a privacy-preserving manner. Under such a setting, multiple clients collaboratively train a global generic model under the coordination of an edge server. But the training efficiency is often throttled by challenges arising from limited communication and data heterogeneity. This paper presents a distributed training paradigm that employs analog over-the-air computation to address the communication bottleneck. Additionally, we leverage a bi-level optimization framework to personalize the federated learning model so as to cope with the data heterogeneity issue. As a result, it enhances the generalization and robustness of each client's local model. We elaborate on the model training procedure and its advantages over conventional frameworks. We provide a convergence analysis that theoretically demonstrates the training efficiency. We also conduct extensive experiments to validate the efficacy of the proposed framework. 

\end{abstract}
\begin{keywords}
Federated learning, personalization, wireless edge network, over-the-air computation, robustness. 
\end{keywords}
\section{Introduction}
\label{sec:intro}
With the increasing concerns on data privacy as well as the rapid growing capability of edge devices, deploying the federated learning (FL)~\cite{MaMMooRam:17AISTATS} at the edge of wireless network, commonly coined as \textit{federated edge learning} (FEEL), is attracting arising attentions~\cite{yang_tcom_fl_scheduling,saad2019survey}, where the computation tasks could be decoupled from the cloud to the edge of the network in a privacy-preserving paradigm.

However, in real-world implementations of the FEEL system, a typical training process of a generic global model requires hundreds of communication rounds among the massively distributed clients. The iterative gradient exchange would bring hefty communication overhead~\cite{MaMMooRam:17AISTATS, li2020flsurvey}. 
Hence, for a digital communication based-FEEL system run over the resource-constrained network, the limited communication bandwidth would inevitably constrain the scalability, since every selected client in each round  requires an assigned orthogonal sub-channel to perform the update~\cite{YanCheQue:21JSTSP,ding2020ota}.

To combat the communication bottleneck, an array of recent studies \cite{ding2020ota,guo2021ota,chen2018ota_iot,AmiGun:20TSP,sery2020analog,chen2022analog} suggest incorporate \textit{analog over-the-air} (A-OTA) computations into the design of FEEL systems, exploiting the superposition property of the multi-access channels for fast and scalable model aggregations.
The adoption of the A-OTA computations with FEEL, termed as \textit{A-OTA-FEEL}, have been demonstrated to have high spectral efficiency, low access latency, enhanced privacy protection, and reduced communication costs~\cite{ding2020ota,liu2020privacyfree,YanCheQue:21JSTSP}, all benefiting from the automatic ``one-shot'' gradient aggregation for model update~\cite{zhu2021otanetwork,guo2021ota}.
Nevertheless, A-OTA computations inevitably introduce the random channel fading and interference into the aggregated gradients, leading to performance degradations such as the slower convergence and instability~\cite{YanCheQue:21JSTSP,sery2020analog}.
Hence, robust training techniques could be adopted to enhance the performance with channel imperfections.

In addition to the inherent channel fading and interference, current approaches for A-OTA-FEEL system design have not addressed the existing discrepancies in both local data statistics and qualities (i.e., data heterogeneity and label noise) due to the diverse preferences, bias, and hardware capabilities of different clients~\cite{xu2022fedcorr,chen2022analog}.
Such discrepancies in clients' datasets can significantly degrade the FL performance. 
More crucially, these discrepancies would even make the single generic global model fail to achieve good generalization and robustness performance on diverse local data~\cite{tan2022pflsurvey,li2021ditto,smith2017mtl}. 
On the other hand, the future intelligent network is envisioned to be able to provide customized services to the clients~\cite{zhou2020service6g,TangLPT20}. It is necessary to address the individuality of the clients in the design of the A-OTA-FEEL system with personalized intelligent services. 

In view of the above challenges, we propose a personalized training framework in the context of the A-OTA-FEEL. The proposed framework provides personalized model training services while still enjoying the benefits of analog over-the-air computations, in which each client would maintain two models (i.e., generic and personalized models) at the local via two different global and local objectives.
We also provide a convergence analysis of the proposed personalized A-OTA-FEEL framework.
Both the theoretical and numerical results validate the gain from the personalization design.

\section{System model}
\label{sec:propose}
We consider a wireless system consisting of one edge server that is attached to an access point and $K$ clients, where the $i$-th device has a local dataset $\mathcal{D}_i$. 
In this system, communications between the clients and the server are taken place over the spectrum.
Each client's goal is to ($a$) learn a statistical model based on its own dataset and ($b$) exploit information from the dataset of other clients and, aided by the orchestration of the server, attain an improvement toward its locally learned model while preserving privacy. 
Such tasks can be achieved via a bi-level optimization based PFL framework. 
More precisely, every client $k$ aims to find a local model $\boldsymbol{v}_k \in \mathbb{R}^d$ that solves the following \textit{personal objective} funciton 
\vspace*{-0.6em}
\begin{align} \label{eq:local_obj}
&\underset{\boldsymbol{v}_k}{\min}\qquad f_k\left( \boldsymbol{v}_k;\boldsymbol{w}^* \right) = F_k\left( \boldsymbol{v}_k \right) +\frac{\lambda}{2}\left\| \boldsymbol{v}_k-\boldsymbol{w}^* \right\| ^2\\ \label{equ:global_obj}
& \text{s.t.} \qquad \boldsymbol{w}^*\in \underset{\boldsymbol{w}}{\argmin}\frac{1}{K}\sum_{i=1}^K{F_i\left( \boldsymbol{w} \right)}
\end{align}
in which $F_i(\cdot): \mathbb{R}^d \rightarrow \mathbb{R}$ denotes the loss function of client $i$, $\boldsymbol{w} \in  \mathbb{R}^d$ is a globally trained generic model, and $\lambda$ is a hyper-parameter that controls the level of personalization of the clients' locally trained personal models.
We use $\eta_l$ to denote the learning rate in the optimization of personal objective. Notably, a large value of $\lambda$ indicates that the clients' local models $\{\boldsymbol{v}_i\}_{i=1}^{K}$ need to well align with the global model $\boldsymbol{w}^*$, promoting commonality across the local models. In contrast, a small $\lambda$ improves personalization. 
Moreover, benefiting from such a bi-level optimization design, the personalized local models $\{\boldsymbol{v}_i\}_{i=1}^{K}$ would have better generalization and robustness performance on the limited local data. 

To solve the above optimization problem, the clients need to not just train their local models through \eqref{eq:local_obj}, but more importantly, jointly minimize a global objective function as per \eqref{equ:global_obj}. 
Due to privacy concerns, the clients will carry out the minimization problem \eqref{equ:global_obj} without sharing data in an FL manner.
The following section presents a model training approach that capitalizes on the properties of analog transmissions for low-latency and high-privacy federated computing.

\begin{figure}[t!]
  \centering{\includegraphics[width=1.00\columnwidth]{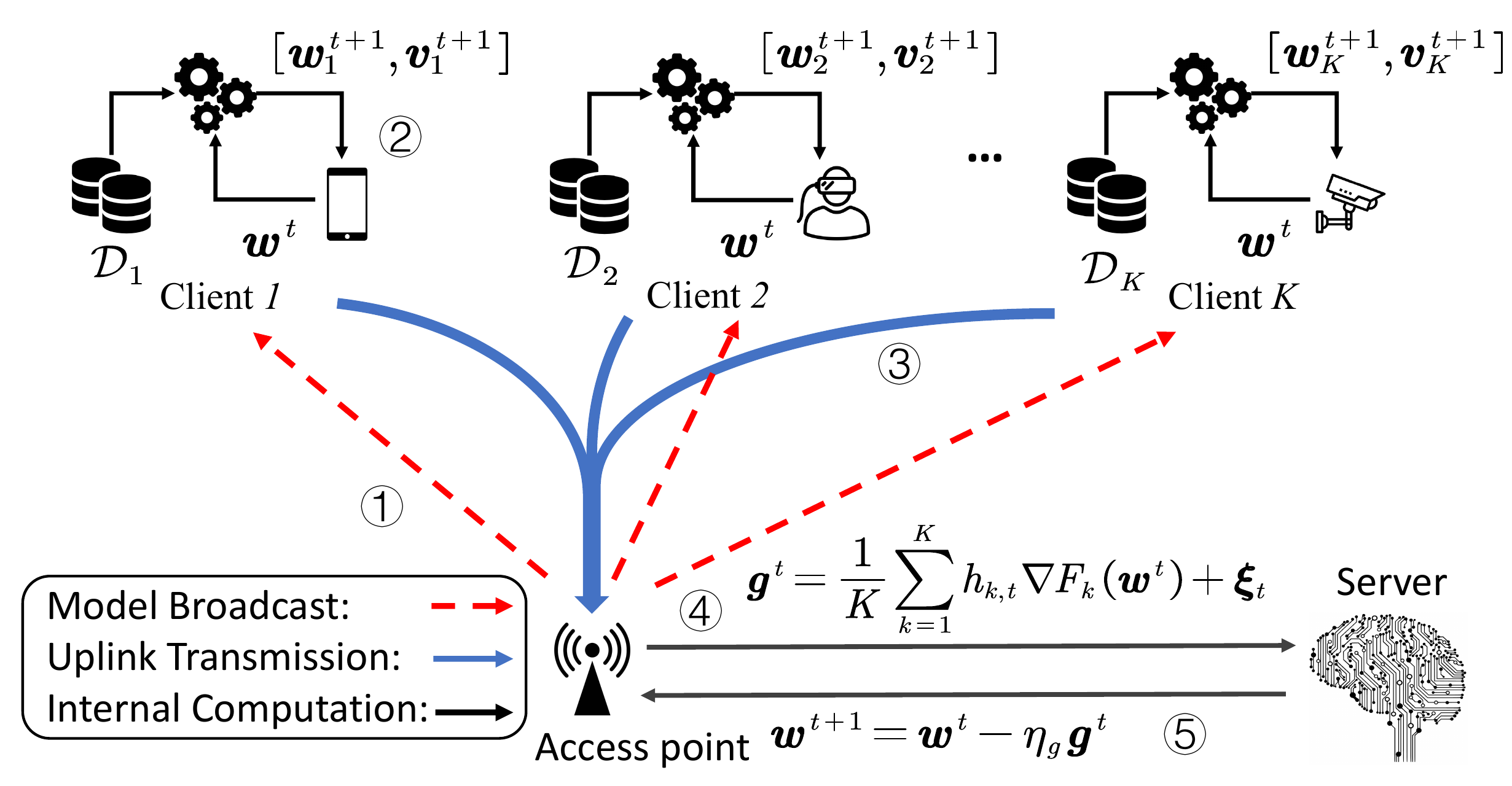}}
  \caption{An overview of personalized analog over-the-air federated edge learning, in which each client maintains a common global model and local personalized model.}  
  \label{fig:system_model}
\end{figure}

\section{ Model Training Procedure}
This section details the PFL model training process based on over-the-air computing schemes. (See Fig. 1 for an overview.) More precisely, we employ A-OTA computations for fast (and highly scalable) gradient aggregation that significantly improves the training efficiency of the global model. The detailed training procedure is elaborated on below. 

\textit{1) Local Model Training:} Without loss of generality, we assume the system has progressed to the $t$-th round of global training, where the clients just received the global model parameters $\boldsymbol{w}^t$ from the edge server. {\footnote{Because of the high transmit power of the access point, we assume the global model can be successfully received by all the clients. }} 
Then, each client $k$ updates its personalized local model $\boldsymbol{v}^t_k$ by optimizing the local personal objective function $f_k( \boldsymbol{v}_k; \boldsymbol{w}^t)$. (For simplicity, we use $\boldsymbol{v}_k$ to denote personal model.) Each client $k$ also computes its local gradient $\nabla F_k( \boldsymbol{w}^t )$ for global model update. 

\textit{2) Analog Gradient Aggregation:} We consider the clients adopt analog transmissions to upload their locally trained parameters. Specifically, once $\nabla F_k( \boldsymbol{w}^t )$ is computed, client $k$ modulates it entry-by-entry onto the magnitudes of a common set of orthogonal baseband waveforms~\cite{YanCheQue:21JSTSP}, forming the following analog signal
\vspace*{-0.6em}
\begin{align}
    x_k(s) = \langle \boldsymbol{u}(s), \nabla F_k( \boldsymbol{w}^t ) \rangle
\end{align}
where $\langle \cdot, \cdot \rangle$ denotes the inner product between two vectors and $\boldsymbol{u}(s) = ({u}_1(s), ..., {u}_d(s)), s \in [0, \tau]$ has its entries satisfying

\begin{align}
    \int_0^\tau u_i^2(s) ds = 1, ~~i = 1, 2, ..., d \\
    \int_0^\tau u_i(s) u_j(s) ds = 0, ~ ~ i \neq j.
\end{align}
\noindent$\tau$ represents the total time of signal duration. Once the analog waveforms $\{ x_k(s) \}_{k=1}^K$ are available, the clients transmit them concurrently to the access point. 
Owing to the superposition property of electromagnetic waves, the signal received at the radio front end of the access point can be expressed as: 
\vspace*{-0.6em}
\begin{align}
    y(s) = \sum_{k=1}^K h_{k,t} P_k x_k(s) + \xi(s),
\end{align}

where $h_{k, t}$ is the channel fading experienced by client $k$, $P_k$ the corresponding transmit power, and $\xi(s)$ denotes the additive noise. 
In this work, we assume the channel fading is i.i.d. across clients, with mean $\mu_h$ and variance $\sigma_h^2$. 
Besides, the transmit power of each client is set to compensate for the large-scale path loss and we use $P$ to denote the average power for all clients.
This received signal will be passed through a bank of match filters, with each branch tuning to $u_i(s), i = 1, 2, .., d$. On the output side, the server obtains:
\vspace*{-0.6em}
\begin{equation}\label{eq:ota_agg}
\boldsymbol{g}^{t}= \frac{1}{K} \sum_{k=1}^{K} h_{k, t} \nabla F_{k}\left(\boldsymbol{w}^{t}\right)+\boldsymbol{\xi}_{t},
\end{equation}
in which $\boldsymbol{\xi}_{t}$ is a $d$-dimensional random vector with each entry being i.i.d. and follows a zero-mean Gaussian distribution with variance $\sigma^2$. 
It is noteworthy that the vector given in \eqref{eq:ota_agg} is a distorted version of the globally aggregated gradient. 

\textit{3) Global Generic Model Update:} Using $\boldsymbol{g}^{t}$, the server updates the global model as follows:
\vspace*{-0.3em}
\begin{equation} \label{eq:global_w_update}
    \boldsymbol{w}^{t+1} \leftarrow \boldsymbol{w}^{t}-\eta_g\boldsymbol{g}^t,
\end{equation}
where $\eta_g$ is the learning rate for generic global model udpate. 
After this, the server  broadcasts the $\boldsymbol{w}^{t+1}$ to all clients for the next round local computing. Such a process will be iterated through multiple rounds until the global model converges.

Notably, the bi-level optimization in the personal model mitigates impacts from the random channel fading and noise introduced by analog over-the-air computations to the globally aggregated gradient, thus improving the robustness of the analog over-the-air federated edge learning system.
Consequently, personalization enhances both the generalization and robustness of the FL system in the presence of data heterogeneity and noisy model aggregation. \footnote{This paper does not consider the architecture-based PFL methods in which each client maintains a personal model with unique architecture via techniques such as sparsification or model weight decoupling~\cite{tan2022pflsurvey}. It would increase the cost of the synchronizations for signal transmission to achieve automatic signal aggregations in the context of A-OTA computations.}

\begin{algorithm}[t!]
    \caption{Personalized A-OTA FEEL framework}
    \label{alg:summary}
    \renewcommand{\algorithmicrequire}{\textbf{Input:}}
    \renewcommand{\algorithmicensure}{\textbf{Output:}}
    \algrenewcommand\algorithmicreq{\textbf{function}}
    \begin{algorithmic}[1]
        \REQUIRE Initial global model $\boldsymbol{w}^0$, initial personal local models $\{\boldsymbol{v}_i\}_{k=1}^{K}$, $T$, $\lambda$,$\eta_g$ 
        \ENSURE Global model $\boldsymbol{w}^T$, personal model $\{\boldsymbol{v}_i\}_{i=1}^{K}$
        \FOR{$t=0,1,2$ \textbf{to} $T-1$}    
            \FOR{$k=1,2,$ \textbf{to} $K$ \textbf{in parallel}} 
                \STATEx \textit{\qquad \# global generic model update}
                \STATE $\nabla F_i(\boldsymbol{w}^{t})$ $\leftarrow$ \textsc{ClientUpdate}($k, \boldsymbol{w}^t$) 
                \STATEx \textit{\qquad \# local personalized model update}
                \STATE Update $\boldsymbol{v}_k$ via solving $f_k\left( \boldsymbol{v}_k;\boldsymbol{w}^t \right)$ 
                \STATE Transmit local gradient $\nabla F_i(\boldsymbol{w}^{t})$ to edge server
            \ENDFOR
            \STATEx \textit{\quad \# Noisy aggregation via analog OTA computations}
            \STATE  $\boldsymbol{g}^{t}= \frac{1}{K} \sum_{k=1}^{K} h_{k, t} \nabla F_{k} \left(\boldsymbol{w}^{t}\right)+\boldsymbol{\xi}_{t}$            
            \STATE $\boldsymbol{w}^{t+1} \leftarrow \boldsymbol{w}^{t}-\eta_g \boldsymbol{g}^t$ \textit{\quad \# Global model update}
            
        \ENDFOR
        \RETURN $\boldsymbol{w}^{T}$,  $\{\boldsymbol{v}_i\}_{k=1}^{K}$
    \end{algorithmic}
\end{algorithm}

We summarize the proposed framework in Algorithm \ref{alg:summary}. 
It is worthwhile to highlight several advantages of the presented framework, including high scalability, low access latency, enhanced privacy, better generalization as well as robustness, brought together by analog over-the-air computations and personalized training.
We would also like to address that we make no assumption about the generic model training, as well as the OTA communication, which indicates that the performance could be further enhanced by advanced federated optimization~\cite{li2020flsurvey} and OTA techniques~\cite{chen2018ota_iot,liu2020privacyfree}.

\section{Convergence analysis}
\label{sec:conv_ana}
This section provides the convergence analysis of our proposed framework from the perspective of both the global model and the local personalized model. 

To facilitate the analysis, we assume that each client's loss function is $\mu$-strongly convex and the local gradient $\nabla F_k(\boldsymbol{w}^t)$ is Lipschitz continuous with constant $L_k>0$. We use $\bar{L}$ to denote the maximal constant among all clients and $L$ is the Lipschitz gradient constant of global objective.
$\delta$ is the diameter of the compact convex parameter set that all model parameters lie in.
We consider that if the global model converges, its convergence rate is denoted by $g\left( t \right)$, i.e., there exists $g\left( t \right)$ that $\underset{t\rightarrow \infty}{\lim}g\left( t \right) =0$ and $E\left[ ||\boldsymbol{w}^t-\boldsymbol{w}^*||^2 \right] \le g\left( t \right)$.
In this work, we denote by $\boldsymbol{v}_k^*$ and $\boldsymbol{z}_k^*$ as $\boldsymbol{v}_{k}^{*}={\mathrm{arg}\min}_{\boldsymbol{v}} f_k\left( \boldsymbol{v};\boldsymbol{w}^* \right)$
and $\boldsymbol{z}_k^* = {\arg\min}_{ \boldsymbol{z} } F_k\left( \boldsymbol{z} \right)$, respectively.
We assume the $l_2$ distance between the optimal local and global model is bounded, i.e., for any $k\in \left[ K \right]$, $\left\| \boldsymbol{z}_{k}^{*}-\boldsymbol{w}^* \right\| \le M$.

We now  present the main theoretical finding of this paper.
First of all, the following theorem provides the convergence rate of the global generic model.
\begin{theorem} \label{thm:ConvRate}
\textit{Under the considered A-OTA FEEL system, let $r_{0}^{2}\triangleq \left\| \boldsymbol{w}^0-\boldsymbol{w}^* \right\| ^2$ be the squared distance between the initial estimate $\boldsymbol{w}_0$ and
$\boldsymbol{w}^*$. If the learning rate $\eta _g$ satisfies
\begin{align} \label{equ:StpSiz_Cdn}
0<\eta _g<\min \left\{ \frac{2}{\mu _h(\mu +L)},\frac{2\mu _h\mu LK}{\sigma _{h}^{2}\bar{L}^2(1+2\delta )(\mu +L)} \right\}, 
\end{align}
then the error of $\boldsymbol{w}^t$ can be bounded as:
\begin{align}
\resizebox{\hsize}{!}{$E\left[ ||\boldsymbol{w}^t-\boldsymbol{w}^*||^2 \right] \le c^tr_{0}^{2}+\frac{\eta _{g}^{2}}{(1-c)}\left( \frac{\sigma _{h}^{2}\delta \bar{L}^2(2+\delta )}{K}+\frac{d\sigma^{2}}{P^2K^2} \right)$ } 
\end{align}
where $0<c\triangleq 1-\frac{2\eta _g\mu _h\mu L}{\mu +L}+\frac{\eta _{g}^{2}\sigma _{h}^{2}\bar{L}^2(1+2\delta )}{K}<1$.
 }
\end{theorem}
\begin{proof}
Please refer to \cite{sery2020analog} for a detailed proof.
\end{proof}

Next, we employ the following lemma to characterize the convergence rate of the local personalized models.
\begin{lemma}
\textit{Under the considered system, let local learning rate satisfy condition \eqref{equ:StpSiz_Cdn}, then the local model of client $k$ converges as:
\begin{align}
& E\left[ ||\boldsymbol{v}_{k}^{t+1}-\boldsymbol{v}_{k}^{*}||^2 \right]
\le \left( 1-\mu \eta _l \right) E\left[ ||\boldsymbol{v}_{k}^{t}-\boldsymbol{v}_{k}^{*}||^2 \right] +\eta _{l}^{2}\lambda ^2M^2\nonumber\\
&+\eta _{l}^{2}\lambda ^2E\left[ ||\boldsymbol{w}^t-\boldsymbol{w}^*||^2 \right]
+2\eta _{l}^{2}\lambda ^2M\sqrt{E\left[ ||\boldsymbol{w}^t-\boldsymbol{w}^*||^2 \right]}\nonumber\\
&+2\eta _l\lambda \sqrt{E\left[ ||\boldsymbol{v}_{k}^{t}-\boldsymbol{v}_{k}^{*}||^2 \right] E\left[ ||\boldsymbol{w}^t-\boldsymbol{w}^*||^2 \right]}.
\end{align}
}
\end{lemma}
\begin{proof}
Please refer to \cite{li2021ditto} for detailed proof.
\end{proof}

Aided by the above result, we obtain the convergence rate of the global model as the following.
\begin{theorem} 
\textit{Under the considered A-OTA FEEL system, if there exists a variable $A$ satisfying $\frac{g(t+1)}{g(t)}\ge 1-\frac{g(t)}{A}$, then, there is a constant $C<\infty $ such that for any client $k$, $E\left[ \left\| \boldsymbol{v}_{k}^{t}-\boldsymbol{v}_{k}^{*} \right\| ^2 \right] \le Cg(t)$ with a local learning rate given by $\eta_l=\frac{2g\left( t \right)}{A\mu}$.
 }
\end{theorem}
\begin{proof}
    We omit the proof due to the space limit. 
\end{proof}
To this end, we can see that via A-OTA computing, both the global and local personalized  models attain linear convergence rates, while addressing the non-ideal gradient updates.

\section{Numerical results}
\label{sec:result}
This section evaluates the performance of our proposed framework.
Particularly, we examine the performance of the personalized local training in terms of generalization power and robustness compared to conventional settings and baselines. 
We also explore the robustness performance of the framework in the context of the noisy local data (i.e., part of the local training data are annotated with wrong labels).

\subsection{Experiment setup}
We evaluate our framework on image classification tasks on CIFAR-10/100~\cite{Krizhevsky_2009} with ResNet-18 and ResNet-34~\cite{he2016resnet}, respectively. Both IID and non-IID data settings are considered, in which the non-IID data partitions are implemented with Dirichlet distribution and the identicalness of the distributions could be controlled by the parameter $\alpha$. Unless otherwise specified, we use $K=100$ for CIFAR-10, $K=50$ for CIFAR-100, and Rayleigh fading with average channel gain $\mu_h=1$. We select $\lambda$ from comparison experiments. The federated label noise setting is the same as the \cite{xu2022fedcorr}\footnote{For all label noise settings, we use lower bound 0.5 for local label noise level. Details can be found in \cite{xu2022fedcorr}.}.

\subsection{Performance evaluation}
\begin{figure}[t!]
\vspace*{-0.6em}

  \centering{\includegraphics[width=0.95\columnwidth]{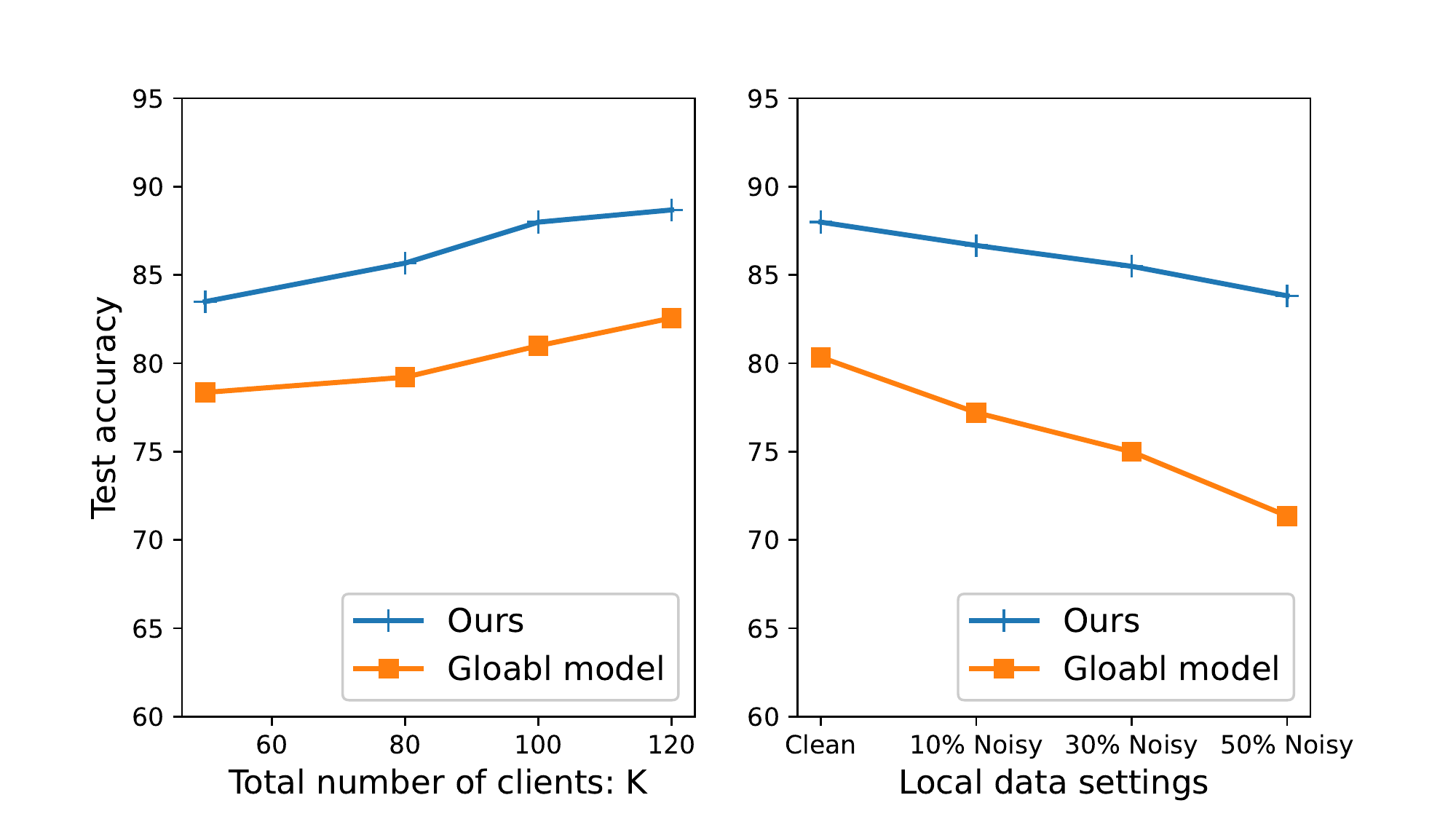}}
  \caption{Performance comparison of the best test accuracy on CIFAR-10 with IID data settings. (Left): Performance with different total number of clients $K$. (Right): Performance with different ratios of clients containing noisy local data.}  
  \label{fig:acc}
\end{figure}

We first compare the personalized models performance of our proposed framework with generic global model from conventional FL setup in Fig.~\ref{fig:acc} with IID local data partition, using the same configurations of A-OTA-FEEL system. The two sub-figures demonstrate the consistent outperformance of personalization training scheme. 
Specifically, increasing total number of clients in the system (i.e., a larger $K$) would improve the system performance for both two settings in A-OTA, and personalized training presents a more robust generalization with diverse local data quality. 

To further demonstrate the outperformance of the proposed framework, we provide the detailed best test accuracy comparison in Tab. 1 on CIFAR-10/100 with non-IID data, compared with FedAvg~\cite{MaMMooRam:17AISTATS}, FedProx~\cite{li2020flsurvey} and FedRep~\cite{collins2021exploiting} with same OTA setup.
In such context, our proposed personalzied training method achieves best test accuracies across all settings, which shows the superiority with respect to the generalization and robustness.

\begin{table}[t!]
\label{tab:results}
\caption{Average (3 trails) of the best test accuracy comparison on CIFAR-10/100 with real-world data settings. The highest accuracy for each setting is boldfaced.}
\centering
\begin{adjustbox}{width=0.95\columnwidth,center}
\begin{tabular}{l|l|ccc}
\hline
\midrule
                
\multirow{2}{*}{} & \multirow{2}{*}{Methods} & \multicolumn{2}{c}{CIFAR-10} & CIFAR-100  \\ \cline{3-5} 
                      &              &$\alpha=10$&$\alpha=1$ & $\alpha=1$  \\ \hline
\multirow{4}{*}{Clean}& OTA-FedAvg   &  76.32    & 72.71     &65.12  \\
                      & OTA-FedProx  &  76.45    & 72.90     &66.35  \\
                      & OTA-FedRep   &  82.44   &  79.93    & - \\
                      & Ours         &  \textbf{83.57}    & \textbf{81.05}     &\textbf{69.33}  \\ \hline
\multirow{4}{*}{Noisy}&  OTA-FedAvg  &  70.51    & 67.15     &58.81  \\
                      &  OTA-FedProx &  72.06    & 69.62     &59.29  \\
                      & OTA-FedRep   & 77.09    &  74.23    & - \\
                      &  Ours        &  \textbf{78.74}    & \textbf{75.31}     &\textbf{63.30} \\ \hline
\end{tabular}
\end{adjustbox}
\end{table}

\section{Conclusion}
\label{sec:conclu}
In this paper, we proposed a personalized A-OTA-FEEL framework that utilizes bi-level optimization and analog transmissions to address the data heterogeneity and communication efficiency challenges. 
Both the theoretical and empirical results were provided to demonstrate the effectiveness of the proposed framework. 
We highlighted the robustness performance of the PFL in edge learning.
To the best of our knowledge, this is the first work that explores the PFL model in A-OTA FEEL systems. We envision that PFL could be a potential technique to provide customized services in future intelligent networks. 
\bibliographystyle{IEEEbib}
\bibliography{strings}

\end{document}